\documentclass[twoside,twocolumn]{article}

\usepackage{natbib}

\usepackage{amssymb,graphicx,url,amsmath,subfigure
}

\usepackage{times}
\usepackage{tikz}
\usepackage{soul}
\usepackage{color}
\usepackage{xcolor}
\usepackage{algorithm}
\usepackage{algorithmic}
    \usepackage{enumerate}

\usepackage{comment}
\usepackage{amsthm}

\usetikzlibrary{shapes.geometric,positioning}

\newcommand{\Exp}{\mathbb{E}}

\newcommand{\R}{{\mathbb R}}

\newcommand{\cX}{{\cal X}}

\newcommand{\bx}{{\bf x}}

\newcommand{\bX}{{\bf X}}

\newcommand{\bmu}{{\boldsymbol \mu}}
\newcommand{\bSigma}{{\bf \Sigma}}

\newtheorem{Theorem}{Theorem}
\newtheorem{Lemma}{Lemma}

\newtheorem{Example}{Example}

\definecolor{MyDarkGreen}{rgb}{0.17,0.46,0.25} 
\definecolor{MyDarkRed}{rgb}{0.88,0.22,0.21} 
\definecolor{MyDarkBlue}{rgb}{0.11,0.11,0.70}

\definecolor{lightgray}{gray}{0.85}
\sethlcolor{Dandelion}

\usetikzlibrary{arrows,shapes,plotmarks,positioning}
\usetikzlibrary{decorations.markings}

\tikzset{>=stealth'} 
\tikzstyle{graphnode} = 
   [circle,draw=black,minimum size=22pt,text centered,text
     width=22pt,inner sep=0pt] 
\tikzstyle{var}   =[graphnode,fill=white]
\tikzstyle{vardashed}   =[graphnode,draw=gray,fill=white]
\tikzstyle{obs}   =[graphnode,fill=black,text=white]
\tikzstyle{obsgrey}   =[graphnode,draw=white,fill=lightgray,text=black]
\tikzstyle{par}    =[graphnode,draw=white,fill=red,text=black] 
 \tikzstyle{crucial} =[graphnode,draw=white,fill=yellow,text=black] 
\tikzstyle{fac}   =[rectangle,draw=black,fill=black!25,minimum size=5pt]
\tikzstyle{facprior} =[rectangle,draw=black,fill=black,text=white,minimum size=5pt]
\tikzstyle{edge}  =[draw=white,double=black,very thick,-]
\tikzstyle{blueedge}  =[draw=white,double=blue,very thick,-]
\tikzstyle{rededge}  =[draw=white,double=red,very thick,-]
\tikzstyle{prior} =[rectangle, draw=black, fill=black, minimum size=
5pt, inner sep=0pt]
\tikzstyle{dirprior} = [circle, draw=black, fill=black, minimum
size=5pt, inner sep=0pt]

\tikzstyle{dot_node}=[draw=black,fill=black,shape=circle]

\date{25 November 2019}

\begin{document}
\frenchspacing

\title{{\bf Feature relevance quantification in explainable AI: \\A causal problem}} 

\author{Dominik Janzing, Lenon Minorics, and Patrick Bl\"obaum\\
{\small Amazon Research T\"ubingen, Germany }\\
{\small \{janzind, minorics, bloebp\}@amazon.com} }

\maketitle

\begin{abstract}
We discuss promising recent contributions on quantifying feature relevance using Shapley values, where we
observed some confusion on 
which probability distribution is the right one for dropped features. We argue that the confusion is based on
not carefully distinguishing between {\it observational} and {\it interventional} conditional probabilities and try a clarification based on Pearl's seminal work on causality.  
We conclude that {\it unconditional} rather than {\it conditional} expectations 
provide the right notion of {\it dropping} features in contradiction to the theoretical justification of the software package SHAP. Parts of SHAP are unaffected because
unconditional expectations (which we argue to be conceptually right) are used as {\it approximation} for the conditional ones, which encouraged others to `improve' SHAP in a way that we believe to be flawed. 
\end{abstract}

\section{Motivation} 
Despite several impressive success stories of deep learning, not only researchers in the field have been shocked more recently
about lack of robustness for algorithms that were actually believed to be powerful.  
Image classifiers, for instance, fail spectacularly once the images are subjected to adversarial changes that appear minor to humans, see e.g. \cite{Goodfellow2015, Sharif2016, Kurakin2017, Eykholt2018, Brown2018}.    
Understanding these failures is challenging since it is hard to analyze which features were decisive for the classification in a particular case.   
However, lack of robustness is only one of several different motivations for getting artificial intelligence {\it interpretable}. 
Also the demand for getting {\it fair} decisions, e.g., \citet{Dwork2012,Kilbertusetal2017,barocas2018}, requires understanding of algorithms.  
In this case, it may even be subject of legal and ethical discussions
{\it why} an algorithm came to a certain conclusion.

To formalize the problem, we describe the input / output behaviour as a function $f:\cX_1,\dots,\cX_n \rightarrow \R$ 
where $\cX_1,\dots,\cX_n$ denote the ranges of some input variables $(X_1,\dots,X_n)=:\bX$ (discrete or continuous), while we assume the target variable $Y$ to be
real valued for reasons that will become clear later. 
Given one particular input $\bx:=(x_1,\dots,x_n)$ we want to quantify to what extent each $x_j$ is `responsible' for the output $f(x_1,\dots,x_n)$.  
This question makes only sense, of course, after specifying what should one input be {\it instead}. Let us first 
consider the case where $\bx$ is compared to some `baseline' element $\bx'$, which has been studied in the literature mostly for the 
case of real-valued inputs and differentiable  $f$. Based on  a hypothetical scenario where only some of the baseline values $x'_j$ are replaced with $x_j$  
while others are kept one wants to quantify to what extent each component $j$ contributes to the difference $f(\bx)-f(\bx')$. 
The focus of the present paper, however, is a scenario where the baseline is defined by the expectation $\Exp[f(\bX)]$ over some distribution $P_\bX$.
To explain the relevance of each $j$ for the difference $f(\bx)-\Exp[f(\bX)]$ one considers a scenario where only some values are kept and
the remaining ones are {\it averaged over some probability distribution}. The main contribution of this paper is
to discuss which distribution is the right one. Recalling the difference between {\it interventional} and {\it observational} conditional distributions in the field of causality,
we explain why we disagree with the interesting proposal of \cite{Lundberg2017} in this regard.
Further we argue that our criticism is irrelevant for any software that `approximates' the conditional expectation
(which we consider conceptually wrong) by the unconditional expectation, as proposed by  \cite{Lundberg2017}.
The paper is structured as follows. Section~\ref{sec:axiom} summarizes  results from the literature regarding axioms for feature attribution 
for the case where there is a unique baseline reference input. Here integrated gradients and Shapley values (as the generalization to discrete input) are the unique 
attribution functions for the stated set of axioms. Section~\ref{sec:average} discusses the attribution problem for the case where one averages over unused features as in \cite{Lundberg2017}, and then we present our criticism. 
We think that the big overlap of the present paper with existing literature is justified by aiming at this clarification only, while keeping this clarification as self-consistent as possible. In particular, the very general discussion of \cite{Datta2016} contains all the ideas of this work at least implicitly, but since it appeared
before \cite{Lundberg2017} it could not explicitly discuss the conceptual problems raised by the latter. Our view on marginalization over unused features is supported by \cite{Datta2016} for similar reasons.  
In Section \ref{numerics} we present different experiments which illustrate our arguments.

\section{Prior Work \label{sec:axiom}}
The growth of deep neural networks recently motivated many researchers to investigate feature attribution,
see e.g. \cite{Shrikumar2016} for DeepLIFT, \cite{Binder2016} for Layer-wise Relevance Propagation (LRP), \cite{Ribeiro2016} for Local Interpretable Model-agnostic Explanations (LIME), and for gradient based methods \cite{Chattopadhyay2019}.  
For a summary of common architecture agnostic methods, see \cite{Molnar2019}. We first discuss two closely related concepts that
arise from an axiomatic approach.

\subsection{Integrated gradient \label{subsec:intgrad}}
\cite{Sundararajan2017}, investigated the attribution of $x_i$ to the difference
\begin{align}
f(\bx) - f(\bx'), \label{difference f(inp) - f(base)}
\end{align}
where $\bx'$ is a given baseline. 
Under the assumption that $f$ is 
differentiable almost everywhere\footnote{see \citet[Proposition 1]{Sundararajan2017}}, they defined the attribution of $x_i$ to 
\eqref{difference f(inp) - f(base)} as
\begin{align*}
\text{IntegratedGrads}_i&(\bx; f):= \\& \hspace{-1.2cm} (x_i - x_i') \int_{\alpha = 0}^1 \frac{\partial f(x' + \alpha (x-x'))}{\partial x_i} ~ d \alpha.
\end{align*}
Contrary to LIME, DeepLIFT and LRP, this attribution method has the advantage that all of the following $5$ properties are satisfied (see \cite{Sundararajan2017} and \cite{AAS2019}):
\begin{enumerate}
\item[\textit{1.}] \textit{Completeness:} If $atr_i(\bx ; f)$ denotes the attribution of $x_i$ to \eqref{difference f(inp) - f(base)}, then
\begin{align*}
\sum_{i} atr_i(\bx ; f) = f(\bx) - f(\bx').
\end{align*}
\item[\textit{2.}] \textit{Sensitivity:} If $f$ does not depend on $x_i$, then $atr_i(\bx;f) = 0$.
\item[\textit{3.}] \textit{Implementation Invariance:\footnote{Note that this axiom is pointless if it refers to properties of {\it functions} rather than properties of {\it algorithms}. We have listed it for completeness and for consistency with the literature.}} If $f$ and $f'$ are equal for all inputs, then
\begin{align*}
atr_i(\bx ; f) = atr_i(\bx; f') ~~~ \text{ for all } ~~ i.
\end{align*} 
\item[\textit{4.}] \textit{Linearity:} For $a,b \in \mathbb{R}$ holds
\begin{align*}
atr_i(\bx; af_1 + bf_2) &= \\a \cdot atr_i(\bx;& f_1) + b \cdot atr_i(\bx; f_2).
\end{align*}
\item[\textit{5.}] \textit{Symmetry-Preserving:} If $f$ is symmetric in component $i$ and $j$ and $x_i = x_j$ and $x_i' = x_j'$, then
\begin{align*}
atr_i(\bx; f) = atr_j(\bx; f).
\end{align*}
\end{enumerate}
Integrated gradients can be generalized by integrating over an arbitrary path $\gamma$ instead of the straight line. This attribution method is called \textit{path method} and the following theorem holds.
\begin{Theorem}{(\cite[Theorem 1]{FriedmanE2004} and \cite[Theorem 1]{Sundararajan2017})} \label{uniqueness theorem for attribution methods} If an attribution method satisfies the properties Completeness, Sensitivity, Implementation Invariance and Linearity, then the attribution method is a convex combination of path methods. Furthermore, integrated gradients is the only path method that is symmetry preserving.
\end{Theorem}
Notice that convex combinations of path methods can also be symmetry preserving even if the attribution method is not given by integrated gradients. \\

\subsection{Shapley values}
To assess feature relevance {\it relative to the average}, \cite{Lundberg2017}  use a concept that relies on first defining an attribution for binary functions,
or, equivalently, functions with subset as input ('set functions'). We first explain this concept and describe in Section~\ref{sec:average} how it solves the attribution relative to the expectation. 
Assume we are given a set with $n$ elements, say $U:=\{1,\dots,n\}$ and a function
\[
g: 2^U \to \R \quad  \hbox{ with }  g(U) \neq 0,\,  g(\emptyset)=0.
\]
We then ask to what extent each single $j\in U$ contributes to $g(U)$.
A priori, the contribution of each $j$ depends on the order in which more elements are included. We can thus define the contribution of $j$, given $T\subseteq U$
by
\[
C(j|T):= g(T \cup \{j\}) - g(T)
\]
(note that it can be negative and also exceed $g(U)$). With
\begin{align}
\phi_i :=    \sum_{T\subseteq U\setminus \{i\}}  \frac{1}{n {n-1 \choose |T|}}   C(i | T). \label{eq:shapley}
\end{align}
it then holds
\begin{align*}
g(U) = \sum_{i=1}^n    \phi_i .
\end{align*}

The quantity $\phi_i$ is called the {\it Shapley value} \citep{Shapley1953} of $i$, which can be considered the average contribution of  $i$ to $g(U)$.
At first glance, Shapley values only solve the attribution problem for binary inputs by canonically identifying subsets $T$ with binary words
$\{0,1\}^n$. To show that Shapley values also solve the above attribution problem, one can simply define a set function by 
\[
g(T):= f_T(\bx_T) - f(\bx'),
\]
for any subset $T\subseteq \{1,\dots,n\}$. Here, $f_T$ is the `simplified' function with the reduced input $\bx_T$ obtained from $f$ when 
all remaining features are taken from the baseline input $\bx'$, that is, $f_\emptyset (\bx_\emptyset) = f(\bx')$. 

Since Shapley Values also satisfy Completeness, Sensitivity, Implementation Invariance and Linearity \citep{AAS2019} with respect to
the binary function defined by the set function $g$, they are given by a convex combination of path methods. Furthermore, Shapley Values with respect to $g$ are Symmetry-perserving, but don't coincide with integrated gradients. \\

Different ways of feature attribution based on Shapley Values were recently investigated by \cite{Sundararajan2019}. Their main consideration is feature relevance relative to an auxiliary baseline, but feature attribution relative to the expectation (according to an arbitrary distribution) is also  mentioned. Furthermore, \cite{Sundararajan2019} already discussed that Shapley Values based on conditional distributions can assign unimportant features non-zero attribution. However, \cite{Sundararajan2019} didn't consider the problem from a causal perspective.

\section{How should we sample the {\it dropped} features? \label{sec:average}}

We now want to attribute the difference between $f(\bx)$ and the expectation $\Exp[f(\bX)]$ to individual features.  Explaining {\it why} the output 
for one particular input $\bx$ deviates strongly from the average output is particularly interesting for understanding `outliers'. 
Let us introduce some notation first. For any $T\subseteq U$ let 
$\Exp[f(\bx_T,\bX_{\bar{T}})|\bX_T =\bx_T]$ denote the conditional expectation of $f$, given  $\bX_T =\bx_T$. By $\Exp[f(\bx_T,\bX_{\bar{T}})]$ we denote the
expectation of $f(\bx_T,\bX_{\bar{T}})$ with respect to the distribution of $\bX_{\bar{T}}$ {\it without conditioning} on $\bX_T=\bx_T$. Let us call this expression
`marginal expectation' henceforth. 

Accordingly, we now discuss two different options for defining `simplified functions' $f_T$ where all features from $\bar{T}$ are dropped:
\begin{align}
f_T(\bx) &:=   \Exp[f(\bx_T,\bX_{\bar{T}})|\bX_T =\bx_T]    \label{eq:o1} \\
\hbox{ {\bf or} }  \quad f_T(\bx) & :=   \Exp[f(\bx_T,\bX_{\bar{T}})]  \quad \hbox{ {\bf ?} } \label{eq:o2}
\end{align}
 \cite{Lundberg2017} propose \eqref{eq:o1}, but since it is difficult to compute they {\it approximate} it by \eqref{eq:o2}, which they justify by the simplifying assumption
of feature independence. 
Using the set function $g(T):=f_T(\bx) - f_\emptyset (\bx)$, they compute Shapley values $\phi_i$ according to \eqref{eq:shapley}. 
We will argue that using \eqref{eq:o2} rather than \eqref{eq:o1} is conceptually the right thing in the first place. Our clarification is  supposed 
to prevent others from `improving' SHAP by finding an approximation for the conditional expectation that is better than the marginal expectation, like, for instance \cite{AAS2019} and  \citep{Lundberg2018}\footnote{Note that TreeExplainer in SHAP has meanwhile been changed accordingly.}


To explain our arguments, let us first explain why marginal expectations occur naturally in the field of causal inference.
\paragraph{Observational versus interventional conditional distributions}
The main ideas of this paragraph can already be found in \cite{Datta2016} in more general and abstract form, see also \cite{Friedman2001} and \cite{Zhao2019}, but we want to rephrase them in a way that optimally prepares the reader to 
the below discussion. Assume we are given the causal structure shown in Figure~\ref{fig:toycs}.
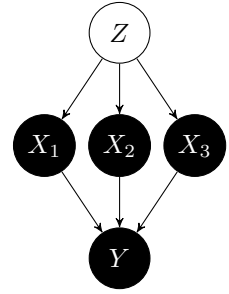
\begin{figure}[h]
   \centerline{ 
    \begin{tikzpicture}
        \node[var] at (2,2.5) (Z) {$Z$}; 
          \node[obs] at (1,1) (X_1) {$X_1$} edge[<-] (Z);
          \node[obs] at (2,1) (X_2) {$X_2$} edge[<-] (Z);
     \node[obs] at (3,1) (X_3) {$X_3$} edge[<-] (Z);
       \node[obs] at (2,-0.5) (Y) {$Y$} edge[<-] (X_1) edge[<-] (X_2) edge[<-] (X_3); 
  \end{tikzpicture}
}
\caption{\label{fig:toycs} A simple causal structure where the observational conditional $p(y|x_1)$ does not correctly describe how $Y$ changes after {\it intervening} on $X_1$ because the common cause $Z$ `confounds' the relation between $X_1$ and $Y$.}
\end{figure}
Further, assume we are interested in how the expectation of $Y$ changes when we manually set $X_1$ to some value $x_1$. This is {\it not} given by $\Exp[Y|X_1=x_1]$ 
because observing $X_1=x_1$ changes also the distribution of $X_2,X_3$ due to the dependences between $X_1$ and $X_2,X_3$ (which are generated by the
common cause $Z$). This way, the difference between $\Exp[Y]$ and $\Exp[Y|X_1=x_1]$  is not only due to the influence of $X_1$, but  can  also be caused by the influence of $X_2,X_3$. The impact of setting $X_1$ to $x_1$ is captured by Pearl's do-operator \cite{Pearl:00} instead, which yields 
\begin{align}
& \Exp[Y|do(X_1=x_1)]  \nonumber \\
= & \int  \Exp[Y|x_1, x_2, x_3] p(x_2,x_3) dx_2 dx_3 \label{eq:do1}.
\end{align}
This can be easily verified using the backdoor criterion \cite{Pearl:00} since (phrased in Pearl's language) the variables $X_2, X_3$ `block the backdoor path' $X_1 \leftarrow Z \rightarrow Y$.  Observations from $Z$ are not needed, we may therefore assume $Z$ to be latent, which we have indicated by white color. 

For our purpose, two observations are important: first,   \eqref{eq:do1} does not contain the conditional distribution, given $X_1=x_1$.
Replacing $p(x_2,x_3)$ with $p(x_2,x_3|x_1)$ in \eqref{eq:do1} would yield the {\it observational} conditional expectation 
$\Exp[Y|X_1=x_1]$, which we are not interested in. In other words, the intervention on $X_1$ breaks the dependences to $X_2,X_3$. 
The second observation that is crucial for us is that the dependences between $X_2,X_3$ are kept, they are unaffected by the intervention on $X_1$. 
\paragraph{Why observational conditionals are flawed} 
Let us start with a simple example.
\begin{Example}[irrelevant feature]
Assume we have
\[
f(x_1,x_2) = x_1.
\]
Obviously, the feature $X_2$ is irrelevant. 
Let both $X_1,X_2$ be binaries and 
\[
p(x_1,x_2) =  \left\{\begin{array}{cc} 1/2  &   \hbox{ for } x_1=x_2 \\ 0 & \hbox{ otherwise } \end{array} \right. .
\]
{\bf (1) with conditional expectations:}
\begin{eqnarray}
f_\emptyset (\bx) &=& \Exp[f(X_1,X_2)]  =  1/2  \label{eq:00}\\
f_{\{1\}}(\bx) &=&  \Exp[f(x_1,X_2)| x_1] = x_1\\
f_{\{2\}}(\bx)  &=& \Exp[f(X_1,x_2)|x_2]  = x_2  \label{eq:11}\\
f_{\{1,2\}}(\bx) &=& f(x_1,x_2) = x_1  
\end{eqnarray} 
Therefore, 
\begin{eqnarray*}
C(2|\emptyset) &=& f_{\{2\}}(\bx) - f_{\emptyset}(\bx) = x_1 - 1/2\\ 
C(2|\{1\})  &= & f_{\{1,2\}}(\bx) -  f_{\{1\}} (\bx) = x_1 - x_1.   
\end{eqnarray*}
Hence, the Shapley value for $X_2$ reads:
\begin{align*}
\phi_2  = \frac{1}{2} \left( x_1-1/2 +x_1 -x_1 \right) 
= x_1/2 - 1/4 \neq 0.  
\end{align*}

{\bf (2) with marginal expectations:}
\begin{eqnarray}
f_\emptyset(\bx) &=& \Exp[f(X_1,X_2)]  =  1/2 \label{eq:00new}\\
f_{\{1\}}(\bx) &=&  \Exp[f(x_1,X_2)] = x_1 \\
f_{\{2\}}(\bx) &=& \Exp[f(X_1,x_2)]  = 1/2\\
f_{\{1,2\}}(\bx) &=& f(x_1,x_2) = x_1 \label{eq:11new}.
\end{eqnarray} 
We then obtain
\begin{eqnarray*}
C(2|\emptyset) &=& f_{\{2\}}(\bx) -f_{\emptyset}(\bx)  = 0\\ 
C(2|\{1\})  &= & f_{\{1,2\}}(\bx) - f_{\{1\}}(\bx) = 0,   
\end{eqnarray*}
which yields $\phi_2=0$.
\end{Example}
The example proves the follow result, which were already discussed in \cite{Sundararajan2019}:
\begin{Lemma}[failure of Sensitivity]\label{lem:main}
When the relevance of $\phi_i$ is defined by defining `simplified' functions $f_T$ via conditional expectations 
\[
f_T(\bx_T):= \Exp[f(\bx) | \bX_{\bar{T}} =\bx_{\bar{T}}],
\]
then $\phi_i\neq 0$ does not imply that $f$ depends on $x_i$.
\end{Lemma} 
The example is particularly worrisome because we mentioned earlier that Shapley values satisfy the axiom of sensitivity, while
Lemma~\ref{lem:main} seems to claim the opposite. 
The resolve this paradox, note that the Shapley values refer to binary functions (or set functions) and
reading \eqref{eq:00} to \eqref{eq:11} as the values of a binary function $\tilde{g}$ with inputs $(z_1,z_2) = 00,10,01,11$ we clearly observe that $\tilde{g}$ depends also on the second bit. This way, the Shapley values do not violate sensitivity for $\tilde{g}$, but we certainly care about `sensitivity for $f$'.
Note that this distinction between the binary function $\tilde{g}$ and $f$ is crucial although in our example $f$ is binary itself. 
Fortunately, the second bit is irrelevant for the binary function $\tilde{g}$  defined by \eqref{eq:00new} and \eqref{eq:11new} and we do not obtain the above paradox. 
 
To assess the impact of changing the inputs of $f$, we now switch to a more causal language and state that we consider the inputs of an algorithm as
{\it causes} of the output.  Although this remark seems trivial it is necessary to emphasize that we are not talking about the causal relation between 
any features in the real world outside the computer (where the attribute predicted by $Y$ may be the cause of the features), but only about causality of this technical input / output system\footnote{Accordingly, $Y$ is the output of the system and not a property of the external world.}. 
To facilitate this view, we formally distinguish between the true features $\tilde{X}_1,\dots,\tilde{X}_n$ 
obtained from the objects and the corresponding features $X_1,\dots,X_n$ plugged into the algorithm. This way, we are able to talk about a hypothetical scenario where the inputs are changed compared to the true features. Let us first consider the causal structure in figure~\ref{fig:cs}, top, where the inputs are determined by the true features. In contrast, figure~\ref{fig:cs}, bottom, shows the causal structure after an intervention on $X_1,X_2$ has adjusted these variables to fixed values $x_1,x_2$. 

We now consider the impact of an hypothetical intervention, which leaves the remaining components  {\it unaffected}. They are therefore sampled from their natural joint distribution {\it without} conditioning.
Similar to the above paragraph, we then obtain 
\begin{align}\label{eq:do}
\Exp [ Y | do(\bX_T =\bx_T) ] = \Exp[f(\bx_T,\bX_{\bar{T}})]. 
\end{align}
Our formal separation between the {\it true} values of the features $\tilde{X}_j$ of some object and the corresponding  {\it inputs} $X_j$ of the algorithms
allows us to be agnostic about the causal relations between the true features in the real world, the fact that
 the inputs $X_1,\dots,X_n$ {\it cause} the output $Y$ is the only causal knowledge needed to compute
\eqref{eq:do}. Since the interventional expectations coincide with the marginal expectations, we have thus justified the use of marginal expectations for the Shapley values from the causal perspective. 
\begin{figure}[H]
   \centerline{ 
    \begin{tikzpicture}[scale=0.9]
       \node[var] at (0,4) (X1) {$\tilde{X}_1$} ;
    \node[var] at (1,4) (X2) {$\tilde{X}_2$} ;
    \node[var] at (2,4) (X3) {$\tilde{X}_3$} ;
     \node[var] at (3,4) (X4) {$\tilde{X}_4$} ;
      \node[var] at (4,4) (X5) {$\tilde{X}_5$} ; 
    \node[obs] at (0,2) (X_1) {$X_1$} edge[<-] (X1);
    \node[obs] at (1,2) (X_2) {$X_2$} edge[<-] (X2) ;
    \node[obs] at (2,2) (X_3) {$X_3$} edge[<-] (X3);
     \node[obs] at (3,2) (X_4) {$X_4$} edge[<-] (X4);
      \node[obs] at (4,2) (X_5) {$X_5$} edge[<-] (X5);
       \node[obs] at (2,0) (Y) {$Y$} edge[<-] (X_1) edge[<-] (X_2) edge[<-] (X_3) edge[<-] (X_4) edge[<-] (X_5); 
       \draw[draw=black] (-1,3.25) rectangle ++(6,2);
       \node[anchor=center] at (2,4.7) {object with features};
  \end{tikzpicture}
  }
  \vspace{1cm}
    \centerline{ 
    \begin{tikzpicture}[scale=0.9]
       \node[var] at (0,4) (X1) {$\tilde{X}_1$} ;
    \node[var] at (1,4) (X2) {$\tilde{X}_2$} ;
    \node[var] at (2,4) (X3) {$\tilde{X}_3$} ;
     \node[var] at (3,4) (X4) {$\tilde{X}_4$} ;
      \node[var] at (4,4) (X5) {$\tilde{X}_5$} ; 
    \node[obs] at (0,2) (X_1) {$x_1$};
    \node[obs] at (1,2) (X_2) {$x_2$};
    \node[obs] at (2,2) (X_3) {$X_3$} edge[<-] (X3);
     \node[obs] at (3,2) (X_4) {$X_4$} edge[<-] (X4);
      \node[obs] at (4,2) (X_5) {$X_5$} edge[<-] (X5);
       \node[obs] at (2,0) (Y) {$Y$} edge[<-] (X_1) edge[<-] (X_2) edge[<-] (X_3) edge[<-] (X_4) edge[<-] (X_5); 
       \draw[draw=black] (-1,3.25) rectangle ++(6,2);
       \node[anchor=center] at (2,4.7) {object with features};
  \end{tikzpicture}  
  }
\caption{\label{fig:cs} Top: Causal structure of our prediction scenario: The output $Y$ is determined by the inputs $X_1,\dots,X_n$. In the usual 
learning scenario these inputs coincide with features $\tilde{X}_1,\dots,\tilde{X}_n$ ob some object, that is $X_j=\tilde{X}_j$. 
Bottom: To evaluate the impact of some inputs, say $X_1,X_2$, for the output $Y$ we consider a hypothetical scenario where we adjust these inputs 
to some fixed values $x_1,x_2$ and sample the remaining inputs from the usual joint distribution $P_{X_3,\dots,X_n}$.}
\end{figure}
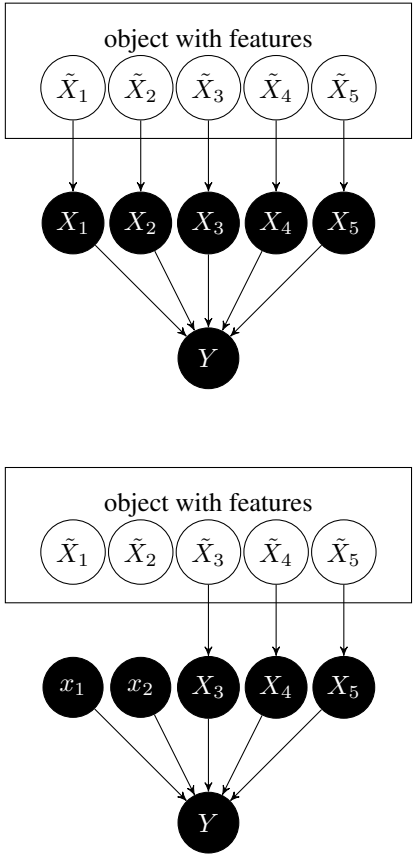
\paragraph{The problem with the symmetry axiom}
We briefly rephrase Example 4.9 of \cite{Sundararajan2019} showing that the symmetry axiom is violated when Shapley values are used for quantifying 
the influence relative to conditional or marginal expectations.
 Figure \ref{fig:lack of symmetry} shows values and probabilities of two random variables $X_1$ and $X_2$ and the values of the function $f(X_1,X_2) =  X_1+X_2$. As explained by \cite{Sundararajan2019}, for the input $(x_1,x_2) = (2,2)$ the value $x_1$ gets attribution $(1-p)$ and $x_2$ gets attribution $(1-q)$.  Therefore, if $p \neq q$, $x_1$ and $x_2$ get different attribution, although $f$ is symmetric. They conclude that this is a violation of symmetry.
Since $X_1$ and $X_2$ are independent, this problem occurs regardless of whether one defines the simplified function $f_T$ with respect
to marginal or conditional expectations. 
One can argue, however, that this result makes intuitively sense because the value $x_j$ that is farther from its mean contributes {\it more} to the 
fact that $f(x_1,x_2)$ deviates from its mean. If we have even $x_1=\Exp[X_1]$, we would certainly say that $x_1$ does not contribute to the deviation from the mean at all. For this reason we do not follow  \cite{Sundararajan2019} in regarding this phenomenon as a problem of this kind of attribution analysis. 
Recall furthermore that we have already mentioned that the symmetry axiom does hold for the corresponding binary function defined by including or not certain features
(simply because symmetry holds for Shapley values). For the above example this binary function is
indeed asymmetric. To check this, define 
\[
\tilde{g}(z_1,z_2) := \Exp [f(\bx_T,\bX_{\bar{T}})],
\]
where $T$ is the set of all $j$ for which $z_j=1$. This function is not symmetric in $Z_1$ and $Z_2$, since we have, for instance,
$\tilde{g}(1,0)= x_1 + \Exp[X_2] \neq \tilde{g}(0,1)=x_2 + \Exp[X_1]$.

\begin{figure}
\begin{tabular}{l l l l}
\hline Probability & $X_1$ & $X_2$ & $f = X_1 + X_2$ \\[2pt] \hline 
$(1-p) \cdot (1-q)$ & 1 & 1  & 2 \\
$(1-p) \cdot q$ & 1 & 2 & 3 \\
$(1-q) \cdot p$ & 2 & 1 & 3 \\
$~p \cdot q$ & 2 & 2 & 4 \\ \hline

\end{tabular}
\caption{\label{fig:lack of symmetry} Table 3 from \cite{Sundararajan2019} which shows an example for alleged lack of symmetry of Shapley Values with respect to the marginal expectation.}
\end{figure}

\section{Numerical Evidence} \label{numerics}
In this section, we show numerically that the marginal expectation $\mathbb{E}[f(\bx_T,\bX_{\bar{T}})]$ is a better choice than $\mathbb{E}[f(\bx_T,\bX_{\bar{T}})|\bX_T = \bx_T]$ to quantify the attribution of each observation $x_j$ of a particular input $\bx = (x_1,\dots, x_n)$ to $f(\bx)- \mathbb{E}f(\bX)$. 

\subsection{Computation of Shapley Values}
As explained by \citet[Section 2.3]{AAS2019}, the implementation of KernelSHAP \citep{Lundberg2017} consists of two parts:
\begin{enumerate}
\item
Using a representation of Shapley Values as the solution of a weighted least square problem for a computationally tractable approximation.
\item
Approximation of $g(T)$.
\end{enumerate}
\subsubsection{Shapley Values as solution of weighted least square problem}
By \cite{Charnes1988}, the Shapley Values to the set function $g$ are given as the solution $(\phi_1, \dots, \phi_n)$ of
\begin{align}
\min_{\phi_1, \dots, \phi_n} \left\{ \sum_{T \subseteq U} \Big[g(T) - \Big(\sum_{j \in T} \phi_j \Big)\Big]^2 k(U,T) \right\}, \label{minimation problem WLS}
\end{align}
where $k(U,T)=(|U| - 1)/(\binom{|U|}{|T|}|T|(|U|-|T|))$ are the \textit{Shapley kernel weights}. Since $k(U,U) = \infty$, we use the constraint $\sum_{j} \phi_j = g(U)$, or, for numerical calculation, we set $k(U,U)$ to a large number.

Since the power set of $U$ consists of $2^n$ elements, the computation time of the Sharpley Values increases exponentially. KernelSHAP therefore samples subsets of $U$ according to the probability distribution induced by the Shapley kernel weights. 
\subsubsection{Approximation of the set function}
As discussed in the previous sections, \cite{Lundberg2017} define
\begin{align*}
f_T(\bx) = \mathbb{E}[f(\bx_T,\bX_{\bar{T}})|\bX_T = \bx_T].
\end{align*}
To evaluate the conditional expectation, they assume feature independence (or weak dependence) to obtain $\mathbb{E}[f(\bx_T,\bX_{\bar{T}})|\bX_T = \bx_T] \approx\mathbb{E}[f(\bx_T,\bX_{\bar{T}})]$ and use the approximation 
\begin{align}
f_{T,\text{KernelSHAP}}(\bx) \approx \frac{1}{K}\sum_k f(\bx_T,\bx_{\bar{T}}^k), \label{KernelSHAP approx}
\end{align}
where $\bx_{\bar{T}}^k$, $k = 1,\dots,K$ are our samples from $\bX_{\bar{T}}$. 
\subsection{Experiments}
To show in an experimental setup that the marginal expectation is a better choice, we consider functions $f$ for which we can calculate analytically the attribution of $x_j$. This is possible for linear functions
\begin{align*}
f(\bx) =  \alpha_0 + \sum_i \alpha_i x_i, ~~~ \alpha_i \in \mathbb{R} 
\end{align*}
since
\begin{align*}
f(\bx) - \mathbb{E}[f(\bX)] = \sum_i \alpha_i (x_i - \mathbb{E}X_i)
\end{align*}
 and hence, the attribution of $x_j$ is  $\alpha_j (x_j - \mathbb{E}[X_j])$.
Our experiments are divided into the following setups:
\begin{enumerate}
\item
We assume that the feature vector $\bX$ follows a multivariate Gaussian distribution.
\item
We use a kernel estimation to approximate the conditional expectation.
\end{enumerate}
For the experiments, we use the KernelExplainer class of the python SHAP package from \citet{Lundberg2017} to calculate Shapley Values with respect to the marginal expectation and the R package SHAPR, in which the methodology of \citet{AAS2019} is implemented, to calculate Shapley Values with respect to the conditional distribution. 

Notice that calculating Shapley Values is also possible for non-linear functions. Further, approximating the marginal expectation is computationally inexpensive compared to the approximation of the conditional expectation with kernel estimation.
\subsubsection{Multivariate Gaussian distribution}
If $\bX \sim N(\bmu,  \bSigma)$ with some mean vector $ \bmu$ and covariance matrix $\bSigma$ , it holds that 
\[
\mathbb{P}(\bX_{\bar{T}}|\bX_T = \bx_T) = N(\bmu_{ \bar{T}|T},\bSigma_{\bar{T}|T})
\]
(see \cite[Section 3.1]{AAS2019}), where
\begin{align*}
\boldsymbol\mu_{\bar{T}|T} &= \boldsymbol\mu_{ \bar{T}} + \bSigma_{\bar{T}T}\bSigma_{T T}^{-1} (\bx_T -  \bmu_T) \\
\bSigma_{\bar{T}|T} &= \bSigma_{\bar{T} \bar{T}} - \bSigma_{\bar{T} T} \bSigma_{T T}^{-1} \bSigma_{T\bar{T}}
\end{align*}
with
\begin{align*}
\bmu =  \left(\begin{array}{c}
      \bmu_T \\
      \bmu_{\bar{T}}
    \end{array}\right), ~~~ \bSigma = \begin{pmatrix}
  \bSigma_{T T} & \bSigma_{T\bar{T}} \\
  \bSigma_{\bar{T}T} & \bSigma_{\bar{T}\bar{T}} \\
 \end{pmatrix}. 
\end{align*}
Hence, we can approximate the conditional expectation by sampling $X_{\bar{T}}$ directly from its distribution.

We simulate Gaussian data and run the experiment for different number of features. For every experiment with multivariate Gaussian distribution, we set the intercept to 0, i.e. $\alpha_0 = 0$.  \\[5pt]
\textbf{Dimension n=3.} In the first 3-dimensional experiment, we let $\alpha_1 = 0$ and choose in every run $\alpha_1$ and $\alpha_2$ independently from the standard normal distribution. Further, we let $\bmu = (0,0,0)^T$ and $\bSigma = c c^T$, where we choose the entries of  $c$ in every run independently from the standard normal distribution and $\bx$ also randomly in every run. The number of runs and the sample size of $\bX$ is 1000. Figure \ref{fig:histogramm3dim} shows the errors $ \phi_j-\text{contr}_j(\bx) $ of the Shapley Values $\phi_j$ with respect to the set function $g(T) = \mathbb{E}[f(\bx_T,\bX_{\bar{T}})] - \mathbb{E}f(\bX)$ (blue) and the set function $g(T) = \mathbb{E}[f(\bx_T,\bX_{\bar{T}})|\bX_T = \bx_T] - \mathbb{E}f(\bX)$ (red). The very precise results for the marginal expectation are mainly from feature 1.
\begin{figure}
\includegraphics[width=0.23\textwidth, height=70px]{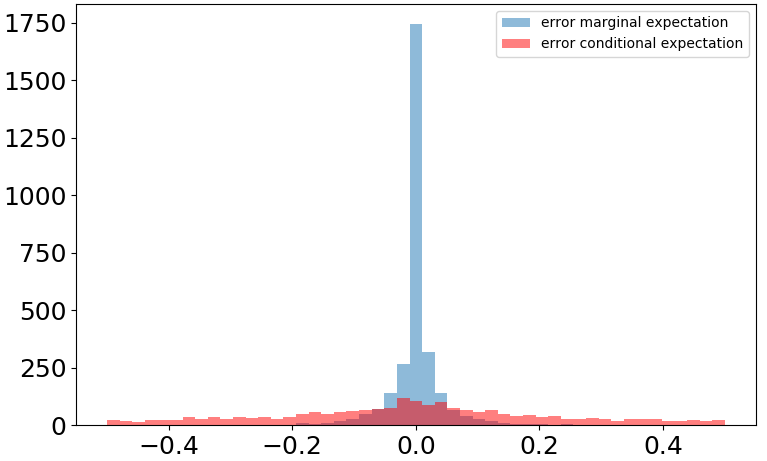}
\includegraphics[width=0.23\textwidth, height=70px]{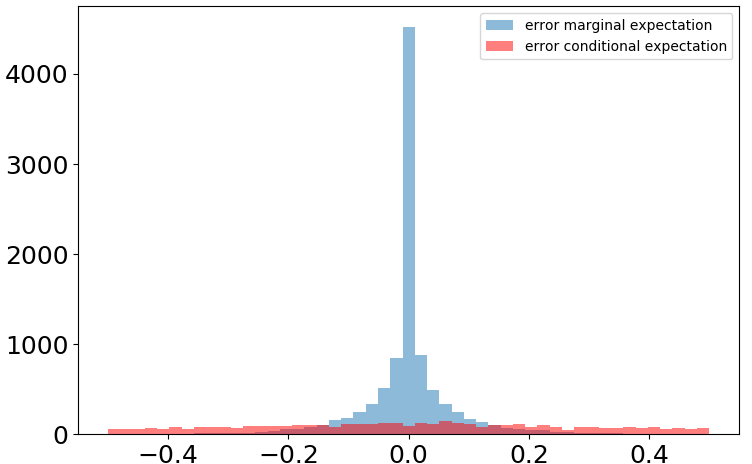}
\caption{\label{fig:histogramm3dim} Histogram showing the error of the Shapley Values for multivariate Gaussian distribution in the 3-dimensional (left) and 10-dimensional (right) setting with $\alpha_1 = 0$. Blue: error using marginal expectation, Red: error using conditional expectation.}
\end{figure}
\\[5pt]
\textbf{Dimension n=10.} In 10-dimensions, we take almost the same setting with the difference that we set the first 3 coefficients to zero, i.e. $\alpha_1=\alpha_2=\alpha_3 = 0$. Again, the very precise results for the marginal expectation are from the features whose coefficient we set to 0.

\subsubsection{Approximation via kernel estimation}
If we have no information about the underlying distribution, it is hard to approximate the conditional distribution sufficiently. However, in low dimensions kernel estimates can provide a good approximation. We take the kernel estimation method from \citet{AAS2019} to show how strongly the Shapley Values w.r.t. conditional expectation deviate from $\alpha_j (x_j - \mathbb{E}[X_j])$. Their approximation is as follows:

\begin{enumerate}
\item
Let $\Sigma_T$ be the covariance matrix of our sample from $\bX_T$. To each point $\bx^i$ of the sample, calculate the Mahalanobis distance (see \cite{Mahalanobis1936})
\begin{align*}
\text{dist}_T(\bx, & \bx^i) := \\ &\sqrt{\dfrac{(\bx_T - \bx^i_T)' \Sigma_T^{-1} (\bx_T-\bx^i_T)}{|T|}},
\end{align*}
where $(\bx_T - \bx^i_T)'$ denotes the transpose of $(\bx_T - \bx^i_T)$.

\item
Calculate the \textit{Kernel weights}
\begin{align*}
w_T(\bx,\bx^i) := \exp\left(- \frac{\text{dist}_T(\bx,\bx^i)^2}{2 \sigma^2}\right).
\end{align*}
Hereby, $\sigma^2 > 0$ is a bandwidth which has to be specified.

\item
Sort the weights $w_T(\bx,\bx^i)$ in increasing order and let $\tilde{\bx}^i$ be the corresponding ordered sampling instances. Then, approximate $g(T)$ by
\begin{align*}
g_{\text{cond}}(T) :=\frac{\sum_{i=1}^K w_T(\bx, \tilde{\bx}^i)f(\bx_{\bar{T}}^i,\bx_T)}{\sum_{i=1}^K w_T(\bx, \tilde{\bx}^i)}.
\end{align*} 
\end{enumerate}

For the experiment, we use the real data set \textit{Human Activity Recognition Using Smartphones Data Set} (see \cite{Anguita2013}) from the UCI repository. The data set consists of 561 features with a training sample size of 7352 and test sample size of  2948. In this experiment, we merge these two samples together and therefore our sample size is 10299. We take randomly 4 features and train a linear model with 3 of these features as inputs and with the 4-th feature as target. We don't consider the label (which is a daily activity performed by the human) of the data set, but the different features have the true label as a common cause. Notice that we are not interested in the quality of the model, but rather in a model for which the ground truth of the attribution is known (because we can certainly look at the linear model obtained). 

Afterwards, we calculate the Shapley Values with SHAP and SHAPR (with $\sigma^2$ set to 0.1 in SHAPR which is the default value) using the first 1000 samples and approximate the expected value $\mathbb{E}X_j$ using the whole data set. The observation $\bx$ is also randomly picked from the data and we run this experiment 1000 times. Figure \ref{fig:histogrammRealData} shows the histogram of the error $\phi_j - \text{contr}_j(\bx)$ for the marginal expectation (blue) and conditional expectation (red). 
\begin{figure}[h]
\centerline{
\includegraphics[width=0.24\textwidth, height=70px]{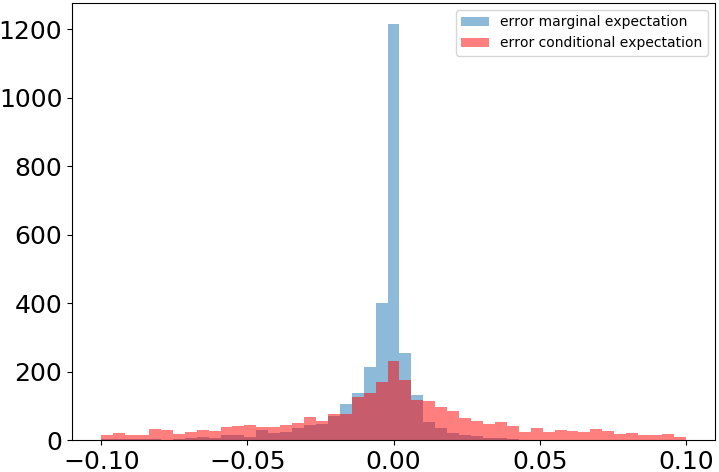}
}
\caption{\label{fig:histogrammRealData} Histogram showing the error of the Shapley Values for the data set \textit{Human Activity Recognition Using Smartphones Data Set}. Blue: error using marginal expectation, Red: error using conditional expectation.}
\end{figure}

\section{Conclusion}
In this work we considered the problem of attributing the output from one particular multivariate input to individual features. We argued that there is a misconception also in recent proposals for feature attribution 
because they use observational conditional distributions rather than interventional distributions. 
Our arguments are phrased in terms of the causal language introduced by \cite{Pearl:00}. 
We argue that parts of the package SHAP from  \citet{Lundberg2017} are unaffected by this misconception (although the corresponding theory part of the paper suffers from this issue) since they `approximates' the observational expectations by an expression that would have been the right one in the first place.  
We think that this clarification is important since other authors tried to `improve' the SHAP package 
in a way that we consider conceptually flawed.
Moreover, we revisited some properties that were stated as desirable in the context of attribution analysis.
If stated in a too vague manner, there is some room for interpretation. We argued, for instance, why we think that our attribution method satisfies a reasonable symmetry property, since attribution via interventional probabilities has been
criticised  for violating alleged desirable symmetry properties. 

\noindent
{\bf Acknowledgements:} The authors would like to thank Scott Lundberg and Anders L{\o}land for their valuable feedback
and Atalanti Mastakouri for remarks on the presentation.

																					
\end{document}